\documentclass[conference,letterpaper]{IEEEtran}

\title{\bf\LARGE Improved Bound for Robust Causal Bandits with Linear Models}

\author{%
 \IEEEauthorblockN{Zirui Yan \qquad Arpan Mukherjee \qquad Burak Var\i c\i \qquad  Ali Tajer \thanks{This work was supported in part by the U.S. National Science Foundation under Grant ECCS-1933107, and in part by the Rensselaer-IBM AI Research Collaboration (AIRC), part of the IBM AI Horizons Network.} }
 \IEEEauthorblockA{Department of Electrical, Computer, and Systems Engineering\\
                   Rensselaer Polytechnic Institute\\
                   Troy, NY, USA\\
                   }
    
}

\usepackage[utf8]{inputenc} 
\usepackage[T1]{fontenc}
\usepackage{url}
\usepackage{ifthen}
\usepackage{cite}
\usepackage[cmex10]{amsmath} 

\usepackage{notations/ISG_style}
\usepackage{notations/ISG_style_reduced}
\usepackage{amssymb}
\usepackage{bbm}

\usepackage{amsmath}
\usepackage{mathtools}
\usepackage{amsthm}

\usepackage{algorithm}
\usepackage{algorithmic}
\usepackage{subcaption}
\usepackage{caption}

\usepackage{xcolor}         
\usepackage{flushend}

\newtheorem{theorem}{Theorem}
\newtheorem{lemma}[theorem]{Lemma}

\newcommand{\algonameLinSEMUCB}{{\rm LinSEM-UCB}}
\newcommand{\algonameRWLinSEMUCB}{{\rm Robust-LCB}}

\IEEEoverridecommandlockouts
\begin{document}

\maketitle

\begin{abstract}
This paper investigates the {\em robustness} of causal bandits (CBs) in the face of temporal model fluctuations. This setting deviates from the existing literature's widely-adopted assumption of constant causal models. 
The focus is on causal systems with linear structural equation models (SEMs). The SEMs and the \emph{time-varying} pre- and post-interventional statistical models are all \emph{unknown} and subject to variations over time. The goal is to design a sequence of interventions that incur the smallest cumulative regret compared to an oracle aware of the entire causal model and its fluctuations. A robust CB algorithm is proposed, and its cumulative regret is analyzed by establishing both upper and lower bounds on the regret. It is shown that in a graph with maximum in-degree $d$, length of the largest causal path $L$, and an aggregate model deviation $C$, the regret is upper bounded by $\tilde{\mathcal{O}}(d^{L-\frac{1}{2}}(\sqrt{T} + C))$ and lower bounded by $\Omega(d^{\frac{L}{2}-2}\max\{\sqrt{T}\; ,\; d^2C\})$. The proposed algorithm achieves nearly optimal $\tilde\mcO(\sqrt{T})$ regret when $C$ is $o(\sqrt{T})$, maintaining sub-linear regret for a broad range of $C$.
\end{abstract}

\section{Motivation and Overview}

In causal systems, 
\emph{interventions} are experimental mechanisms that facilitate uncovering the cause-effect relationships in the networks and distinguishing them from the conventional association measures \cite{pearl2009causality}. Sequential design of interventions is an analytical framework for designing these interventions in a data-adaptive fashion, resulting in an overall reduced experiment cost and a faster process for forming inferential decisions. In this paper, we investigate the problem of causal bandits (CB), which is an experimental design framework that models each intervention as an arm of a multi-armed bandit instance. The goal is to maximize a notion of utility for the causal network. As customary in the causal bandit literature, the utility is chosen as the average value of a leaf node, which we denote as the \textit{reward node}. 

\vspace{.08 in}
\noindent \textbf{Motivation.} We investigate CBs from a new perspective. The existing studies all focus on having a \emph{fixed causal model over time}. In reality, however, large complex causal systems undergo model fluctuations caused by a wide range of reasons such as non-stationarity in the system or heterogeneous data~\cite{huang2017behind, zhang2017causalA}, 
measurement errors~\cite{zhang2017causalB}, selection bias~\cite{zhang2016identifiability}, and missing data~\cite{tu2019causal}. Temporal model fluctuations can change the causal structure or the statistical models. For instance, in drug discovery, there are multiple observable variables or representation nodes~\cite{li2017learning}, and the model fluctuations due to measurement errors can occur in both the observable variables or their weights to the representation nodes. However, algorithms designed for time-invariant settings can be highly susceptible to model fluctuations. For instance, the CB algorithm in \cite{varici2022causal} enjoys a nearly optimal growth in the horizon $T$, i.e., $\mcO(\sqrt{T})$. Nevertheless, it deviates from the optimal rate even in the presence of small model discrepancies, as noted in~\cite{yan2023robust}.  More specifically, the regret of that becomes linear in $T$ if the system experiences model deviations in $T^{\frac{1}{2L}}$ instances, where $L$ is the longest causal path in the graph. Even for small values of $L$, $T^{\frac{1}{2L}}$ will be an extremely small fraction of the instances. For instance, when $T=10^5$ and $L=3$, model deviations in $6$ out of $10^5$ instances result in a linear regret.

\vspace{.08 in}
\noindent \textbf{Contributions.} Our contributions in this paper are three-fold. (1) Under relevant measures of model fluctuations, we design a robust CB algorithm to model deviations over time. (2) We characterize almost-matching upper and lower bounds on the regret as a function of model deviation level, time horizon, and graph parameters. 
Specifically, when considering a \emph{known} budget $C$ that captures the level of model deviations over time, the achievable regret is $\tilde{\mathcal{O}}(d^{L-\frac{1}{2}}(\sqrt{T} + C))$, where $d$ is the maximum in-degree in the graph and $L$ is the length of the largest causal path. The established lower bound is of the order $\Omega(d^{\frac{L}{2}-2}\max\{\sqrt{T}\; , \; d^2C\})$, which has a similar scaling behavior as the upper bound in terms of the graph parameters and the horizon. Note that our algorithm circumvents the exponential growth of the achievable regret with the cardinality of the intervention space, which is exponential in the number of nodes in the graph. (3) Our regret upper bound is tighter than that of \cite{yan2023robust}, which is $\tilde \mcO(d^{L-\frac{1}{2}}\sqrt{NT}+NC)$. When there is no deviation, our regret reduces to $\tilde \mcO(d^{L-\frac{1}{2}}\sqrt{T})$, which is also tighter than the bound in \cite{varici2022causal} by a $\sqrt{N}$ factor.

\vspace{.07 in}
\noindent \textbf{Related literature. } The extent of information available about the causal model critically influences the design of CB algorithms. Broadly, there are two central pieces of information: the causal structure (topology) and the data's pre- and post-intervention statistical models. The earlier studies on CB algorithms assume that both the graph structure and the interventional distributions are known (fully or partially)~\cite{lattimore2016causal,bareinboim2015bandits,sen17,lu2020regret,nair2021budgeted,wei2023approximate}. More recent studies dispense with one or both of the assumptions. Investigations dropping both assumptions include~\cite{lu2021causal,de2022causal,bilodeau2022adaptively,feng2023combinatorialwioutgraph,konobeev2023causal,malek2023additive}. However, these investigations either lack regret guarantees or their regret scales with the cardinality of the action set. Besides,~\cite{bilodeau2022adaptively} investigates the setting in which the distributions are known, while the graph topology is unknown. In this paper, we focus on the setting where the graph structure is known while the distributions are unknown. Prior studies on this setting include~\cite{yabe18causal,maiti2022causal,xiong2022pure,feng2023combinatorial,sawarni2023learning}, which focus on binary random variables. More recently, \cite{varici2022causal} focuses on linear systems and generalizes the results to soft interventions, continuous random variables, and arbitrarily large intervention spaces. In parallel, \cite{sussex2022model} uses soft interventions and generalizes to non-linear models but limits to the Gaussian process SEMs in reproducing kernel Hilbert space (RKHS) and intervention space on controllable action variables. Finally, we note that even though we focus on linear SEMs, our reward is a non-linear function of the unknown parameters. Hence, our CB model fundamentally differs from linear bandits. This is the case even in the CB settings with a fixed model~\cite{varici2022causal}. Nevertheless, we briefly comment on the literature on linear bandits with model misspecification or contamination. These studies assume fixed (permanent) deviation, whereas, in our setting, the deviations can vary over time~\cite{ghosh2017misspecified,lattimore2020learning,foster2020adapting,krishnamurthy2021adapting}. Furthermore, in linear bandits with contamination, the contamination is imposed on the observed rewards~\cite{li2019stochastic,bogunovic2021stochastic,lee2021achieving,zhao2021linear,wei2022model,he2022nearly}, whereas we focus on model deviation.

\vspace{.08 in}
\noindent \textbf{Notations.}
For $N\in \Z_{+}$, we define the set $[N]\triangleq\{1,\cdots,N\}$. The Euclidean norm of a vector $\bX\in\R^N$ is denoted by $\norm{\bX}$. For a subset $\mcS \subseteq [N]$, we define $\bX_{\mcS} \triangleq \bX\odot\mathbf{1}({\mcS})$, where $\odot$ denotes the Hadamard product and the vector $\mathbf{1}({\mcS})\in \{0,1\}^N$ has $1$s at the indices corresponding to $\mcS$. We denote the $i$-th column of matrix $\bA\in\R^{m\times n}$ by $[\bA]_i$, and the entry at $i$-th row and $j$-th column by $[\bA]_{i,j}$. The spectral norm of a matrix is denoted by $\norm{\bA}$. We define the $\bA$-norm for positive semidefinite matrix $\bA$ as $\|\bX\|_{\bA}=\sqrt{\bX^{\top}\bA\bX}$.

\section{Causal Bandit Model}
\label{sec:intro}
\noindent \textbf{Causal graphical model.} Consider a directed acyclic graph (DAG) denoted by $\mathcal{G}(\mathcal{V},\mathcal{E})$, where $\mathcal{V}=[N]$ denotes the set of nodes, and $\mathcal{E}$ denotes the set of edges, and the ordered tuple $(i,j)\in\mathcal{E}$ specifies a directed edge from $i$ to $j$. Each node $i\in[N]$ is associated with a random variable $X_i$. Accordingly, we define the vector $\mathbf{X}\triangleq [X_1,\cdots,X_N]^\top$. We consider a {\em linear} SEM, according to  which
\begin{equation}
\label{eq:linear model}
    \bX=\bB^{\top} \bX+\bepsilon \ ,
\end{equation}
where $\bB\in\R^{N\times N}$ is a strictly upper-triangular edge weight matrix, and $\bepsilon\triangleq (\epsilon_1,\cdots,\epsilon_N)^\top$ denotes the exogenous noise variables, with a known mean $\bnu \triangleq \E[\bepsilon]$. The noise vector $\bepsilon$ is $1$-sub-Gaussian, and its Euclidean norm is upper bounded by $\norm{\bepsilon}\leq m_{\bepsilon}$. The graph's structure is assumed to be \emph{known}, while the weight matrix $\bB$ associated with the graph is \emph{unknown}. For any node $i\in[N]$, we denote the set of parents of~$i$ by ${\rm pa}(i)$. We denote the maximum in-degree of the graph by $d \triangleq \max_i\{|{\rm pa}(i)|\}$ and the length of the longest directed path in the graph by $L$.

\vspace{.08 in}
\noindent\textbf{Intervention model.} We consider \emph{soft interventions} on the graph nodes. A soft intervention on node $i\in[N]$ alters the conditional distribution of $X_i$ given its parents $\bX_{\rm pa}(i)$, i.e., $\P(X_i|\bX_{\rm pa}(i))$. An intervention can be applied to a subset of nodes simultaneously. If node $i\in\mcV$ is intervened, the impact of the intervention is a change in the weights of the edges incident on node $i$. These weights are embedded in $[\bB]_i$, i.e., the $i$-th column of $\bB$. We denote the post-intervention weight values by $[\bB^*]_i\neq [\bB]_i$. Accordingly, corresponding to the interventional weights, we define the interventional weight matrix $\bB^*$, composed of the columns $\{[\bB^*]_i:i\in[N]\}$. Note that soft interventions subsume commonly used stochastic \emph{hard} interventions in which a hard intervention on node $i$ sets $[\bB^*]_i=\boldsymbol{0}$.

Since we allow any arbitrary combination of nodes to be selected for concurrent intervention, there exist~$2^N$ interventional actions to choose from. We define $\mcA\;\triangleq\; 2^{\mcV}$ as the set of all possible interventions, i.e., all possible subsets of $[N]$. For any intervention $a\in\mcA$, we define the post-intervention weight matrix $\bB_a$ such that columns corresponding to the non-intervened nodes retain their observational values from $\bB$, and the columns corresponding to the intervened nodes change to their new interventional values from $\bB^*$. The columns of $\bB_a$ are specified by
\begin{equation}
\label{eq:Ba_construct}
    [\bB_a]_i\;\triangleq\;[\bB]_i\cdot\mathbbm{1}{\{i\notin a\}} + [\bB^*]_i\cdot\mathbbm{1}{\{i\in a\}}\ , 
\end{equation}
where $\mathbbm{1}$ denotes the indicator function. The interventions change the probability models of $\bX$. We define~$\P_a$ as the probability measure of $\bX$ under intervention $a\in\mcA$. For any given $\bB$ and $\bB^*$ we assume that $\norm{[\bB_a]_i} \leq m_B$ for all $a\in\mcA$. Without loss of generality, we assume $m_B = 1$. Due to the boundedness of noise $\bepsilon$ and column of $\bB_a$ matrices, there exists $m\in \R^+$ such that $\norm{\bX}\leq m$.

\vspace{.08 in}
\noindent\textbf{Causal bandit model.} Our objective is the sequential design of interventions. The set of possible interventions can be modeled as a multi-armed bandit setting with $2^N$ arms, each arm corresponding to each possible intervention. Following the canonical CB model~\cite{lattimore2016causal,lu2020regret}, we designate node $N$ (i.e., the node without a descendant) as the \emph{reward} node. Accordingly, $X_N$ specifies the reward value. We denote the expected reward collected under intervention $a\in\mcA$ by
\begin{equation}
\label{equ:mua}
    \mu_a \;\triangleq\; \E_a[X_N]\ ,
\end{equation}
where $\E_a$ denotes expectation under $\P_a$. We denote the intervention that yields the highest average reward by $a^* \triangleq \argmax_{a\in\mcA} \mu_a$; denote the sequence of interventions by $\{a(t)\in\mcA: t\in\N\}$; and denote the data generated at time $t$ and under intervention $a(t)$ by $\bX(t) = [X_1(t), \cdots, X_N(t)]^\top$. The learner's goal is to minimize the average cumulative regret over the time horizon $T$ with respect to the reward accumulated by an oracle aware of the systems model, interventional distributions, and model fluctuations. We define the expected accumulated regret as
\begin{equation}
    \E \left[R(T)\right]\;\triangleq\; T\mu_{a^*} - \sum\limits_{t=1}^T \E[X_N(t)]\ .
    \label{equ:regret}
\end{equation}
\section{Temporal Model Fluctuations}
Due to the size and complexity of the graphical models that represent complex systems, assuming that the observational and interventional models $\bB$ and $\bB^*$ remain unchanged over time is a strong assumption. These models can undergo temporal variations due to various reasons, such as model misspecifications, stochastic behavior of the system, and adversarial influences. To account for such variations, we refer to $\bB_{a(t)}$ as the nominal model of the graph at time $t$ and denote the actual time-varying \emph{unknown} model by $\bD_{a(t)}$. Accordingly, we define the deviation of the actual model from the nominal model as
\begin{equation}
 \Delta_{a(t)}  \;\triangleq\; \bD_{a(t)} - \mathbf{B}_{a(t)} \ . 
\end{equation}
Clearly, if the model of node $i$ undergoes deviation at time $t$ under intervention $a$, we have $\big\|\left[\Delta_{a(t)}\right]_i\big\|\neq 0$. 

We consider the aggregate deviation measure that quantifies the aggregate deviation over time and provides a budget for the maximum deviation in the linear model that model deviations can inflict over time. Specifically, we define the maximum aggregate deviation as
\begin{equation}
C \;\triangleq\; \max_{i\in[N]} \sum_{t=1}^T\max_{a(t)\in\mathcal{A}}\left\|\left[\Delta_{a(t)}(t)\right]_i\right\|\ .
\end{equation} 
This measure of deviation is standard in stochastic bandits~\cite{he2022nearly}, where the deviation budget is defined as the maximum deviation in the reward that the adversary can inflict over time.   We assume that the model deviation budgets specified by $C$ are known to the learner, allowing the CB algorithm to adapt to the varying levels of model deviation.

\section{\algonameRWLinSEMUCB{}~Algorithm}
\label{sec:algo}

In this section, we present the details of our algorithm and provide the performance guarantee in Section~\ref{sec:regret}. 

\vspace{.08 in}
\noindent\textbf{Algorithm overview.} Identifying the best intervention hinges on determining the distributions $\{\P_a:a\in\mcA\}$ that maximizes the expected reward. Nevertheless, these $2^N$ distributions are unknown. Therefore, a direct approach entails estimating these probability distributions, the complexity of which grows exponentially with $N$. To circumvent this, we leverage the fact that specifying these distributions has redundancies since all depend on the observational and interventional matrices $\bB$ and $\bB^*$. These matrices can be fully specified by $2Nd$ scalars, where $d$ is the maximum degree of the causal graph. Hence, at its core, our proposed approach aims to estimate these two matrices. 

We design an algorithm that has two intertwined key objectives. One pertains to the \emph{robust estimation} of matrices $\bB$ and $\bB^*$ when the observations are generated by the non-nominal models. For this purpose, we design a weighted ordinary least squares (W-OLS) estimator. The structure of the estimator and the associated confidence ellipsoids for the estimates are designed to circumvent model deviations effectively. The second objective is designing a decision rule for the sequential selection of the interventions over time. This sequential selection, naturally, is modeled as a multi-armed bandit problem. Therefore, we design an upper confidence bound (UCB)--based algorithm for the sequential selection of the interventions over time. Next, we present the details of the {\bf Robust} {\bf L}inear {\bf C}ausal {\bf B}andit (\algonameRWLinSEMUCB). The steps involved in this algorithm are summarized in Algorithm~\ref{alg:weighted_ucb_algorithm}.

\begin{algorithm}[ht]
\caption{\algonameRWLinSEMUCB}
\label{alg:weighted_ucb_algorithm}
\begin{algorithmic}[1]
   \STATE {\bfseries Input:} Horizon $T$, causal graph $\mcG$, action set $\mcA$, mean noise vector $\bnu$, deviation budget  $C$
   \STATE {\bfseries Initialization}: Initialize 
   $\bB(0) = \bB^{*}(0) =\mathbf{0}_{N \times N}$ and $\bV_{i}(0)= \bV^{*}_{i}(0) = \bI_N, \ \forall i \in [N] $. 
   \FOR{$t = 1, 2,\ldots,T$}
   \STATE Compute ${\rm UCB}_a(t)$ according to \eqref{eq:ucb_definition} for $a\in \mcA$.
   \STATE Pull $a(t) = \argmax_{a \in \mcA} {\rm UCB}_a(t)$ and observe $\bX(t)\!=\!(X_1(t),\dots, X_N(t))^\top$. 
   \FOR{$i \in \{1,\dots,N\}$}
        \STATE Set $w_{i}(t)$ as \eqref{equ:weightsAC}, update $[\bB(t)]_{i}$ according to \eqref{eq:estimate_obs_weighted} and update $[\bB^{*}(t)]_{i}$ according to \eqref{eq:estimate_int_weighted}.
   \ENDFOR
   \ENDFOR
\end{algorithmic}
\end{algorithm}

\vspace{.1 in}
\noindent\textbf{Countering model deviations.} Our approach to circumventing model deviations is to identify and filter out the samples generated by the non-nominal models. We refer to these samples as \emph{outlier samples}. This facilitates forming estimates for $\bB$ and $\bB^*$ based on the samples generated by the nominal models. Since the model deviations may happen on multiple nodes simultaneously, the \algonameRWLinSEMUCB\ is designed to identify the nodes undergoing deviation over time and discard the outlier samples generated by these nodes. Such filtration is implemented via assigning time-varying and data-adaptive weights to different nodes such that the weight assigned to node $i\in[N]$ at time $t\in \N$ balance two factors: the probability of node $i\in[N]$ undergoing deviation at $t$ and the contribution of that sample to the estimator. These weights, subsequently, control how the samples from different nodes contribute to estimating $\bB$ and $\bB^*$.

\vspace{.08 in}
\noindent\textbf{Robust estimation.} We design the {\em weighted} ordinary least squares (OLS) estimators for $\bB$ and $\bB^*$, which at time $t\in\N$ are denoted by $\bB(t)$ and $\bB^*(t)$, respectively. To estimate the observational weights $[\bB]_i$, we use the samples from instances at which node $i$ is not intervened. Conversely, to estimate the interventional weights $[\bB^*]_i$, we use the samples from the instances at which node $i$ is intervened. By defining $\{w_i(t)\in\R_+:i\in[N]\}$ as the set of weights assigned to the nodes at time $t\in\N$, $i$-th columns of these estimates are specified as follows. 
\begin{align}
 [\bB(t)]_{i} &  \triangleq  [\bV_{i}(t)]^{-1} \!\!\!\!\!\!  \sum_{s \in [t] : i \notin a(s)} \!\!\!\!\! w_i(s)\bX_{\Pa(i)}(s) (X_i(s)-\nu_i) \ , 
\label{eq:estimate_obs_weighted} \\
\hspace{-0.2in}[\bB^{*}(t)]_i &  \triangleq [\bV^{*}_{i}(t)]^{-1}  \!\!\!\!\!\! \sum_{s \in [t] : i \in a(s)} \!\!\!\!\!  w_i(s)\bX_{\Pa(i)}(s) (X_i(s)-\nu_i) \ , \label{eq:estimate_int_weighted}
\end{align}
where we have defined the \emph{weighted Gram matrices} as
\begin{align}
\bV_{i}(t) & \;\triangleq \sum_{s \in [t] : i \notin a(s)} w_i(s) \bX_{\Pa(i)}(s) \bX_{\Pa(i)}^\top(s) + \bI_N \ , \label{eq:define_V_obs_weighted} \\
\bV^{*}_{i}(t) & \;\triangleq\sum_{s \in [t] : i \in a(s)} w_i(s) \bX_{\Pa(i)}(s) \bX_{\Pa(i)}^\top(s)+\bI_N \  . \label{eq:define_V_int_weighted}
\end{align}
Furthermore, we define the matrices associated with the squared weights as
\begin{align}
\widetilde{\bV}_{i}(t) & \;\triangleq \sum_{s \in [t] : i \notin a(s)} w_i^2(s) \bX_{\Pa(i)}(s) \bX_{\Pa(i)}^\top(s) + \bI_N \ , \label{eq:define_V_obs_weighted_tilde} \\
\widetilde{\bV}^{*}_{i}(t) & \;\triangleq\sum_{s \in [t] : i \in a(s)} w_i^2 (s)\bX_{\Pa(i)}(s) \bX_{\Pa(i)}^\top(s)+\bI_N \  . \label{eq:define_V_int_weighted_tilde}
\end{align}
Similarly to \eqref{eq:Ba_construct}, we denote the relevant and Gram matrices for node $i$ under intervention $a \in \mcA$ by
\begin{align}
     \widetilde{\bV}_{i,a}(t) &\;\triangleq \;\mathbbm{1}{\{i \in a\}} \widetilde{\bV}^{*}_{i}(t) +  \mathbbm{1}{\{i \notin a\}} \widetilde{\bV}_{i}(t) \ . \label{eq:V_ita(t)ilde}
\end{align}

\vspace{.04 in}
\noindent\textbf{Confidence ellipsoids.} After performing estimation in each round, we construct the confidence ellipsoids for the OLS estimators $\{\mcC_{i}(t):i\in[N]\}$ for the observational weights and $\{\mcC^*_{i}(t):i\in[N]\}$ for the interventional weights 
\begin{align}
    &\hspace{-0.1 in}\mcC_{i}(t)  \; \triangleq \; \Big\{\theta \in \mcB_1  : \label{eq:conf_obs} \\
    &\quad \big\|\theta - [\bB(t-1)]_{i}\big\|_{\bV_i(t-1) [\widetilde{\bV}_i(t-1)]^{^{-1}} \bV_i(t-1)}  \leq \beta_{t} \Big\}  ,  \nonumber \\
    &\hspace{-0.1 in} \mcC^*_{i}(t)  \;\triangleq\; \Big\{ \theta \in \mcB_1: \label{eq:conf_int} \\ 
    &\quad \big\|\theta - [\bB^{*}(t-1)]_{i}\big\|_{\bV^*_i(t-1) [\widetilde{\bV}^*_i(t-1)]^{^{-1}} \bV^*_i(t-1)} \leq \beta_{t} \Big\}  , \nonumber
\end{align}
where $\mcB_1$ is the unit ball in $\R^N$ and $\{\beta_{t}\in\R_+, t\in\N\}$ is a sequence of non-decreasing confidence radii that control the size of the confidence ellipsoids, which we will specify. Accordingly, we define the relevant confidence ellipsoid for node $i$ under intervention $a\in\mcA$ as 
\begin{equation}
    \mcC_{i,a}(t) \;\triangleq\;  \mathbbm{1}{\{i \in a\}}\  \mcC^{*}_{i}(t) + \mathbbm{1}{\{i \notin a\}}\  \mcC_{i}(t) \ . \label{eq:conf_C_iat}
\end{equation}

\vspace{.04 in}
\noindent\textbf{Weight designs.} Designing the weights $\{w_i(t):i\in [N] \}$ at time $
t$ is instrumental in effectively winnowing out the outlier samples. We select the weights that bring the confidence radius $\beta(t)$ down to nearly constant
\begin{equation}
    w_i(t)\;\triangleq\; \min\bigg\{\frac{1}{C} \; , \; \frac{1}{C\norm{\bX_{\Pa(i)}(t)}_{[\widetilde{\bV}_{i,a(t)}(t)]^{^{-1}}}}\bigg\}\ ,\label{equ:weightsAC}
\end{equation}
where the weights are inversely proportional to the norm $\|\bX_{\Pa(i)}(t)\|_{[\widetilde{\bV}_{i,a(t)}(t)]^{^{-1}}}$ and deviation budget $C$ , and they are truncated at $1/C$, which ensures that the weights are not arbitrarily large. We refer the term $\|\bX_{\Pa(i)}(t)\|_{[\widetilde{\bV}_{i,a(t)}(t)]^{^{-1}}}$ as {\em weighted exploration bonus}. A higher exploration bonus means lower confidence in the sample. Setting the weights as the inverse of the exploration bonus avoids potentially significant regret caused by both the stochastic noise and model deviations. We scale the weights proportional to $1/C$ to use smaller weights when the model deviation level is higher.

\vspace{.08 in}
\noindent\textbf{Intervention selection.} We adopt a UCB-based rule for sequentially selecting the interventions. 
Specifically, at each time $t$, our algorithm selects the intervention that maximizes a UCB, defined as the maximum value of expected reward when the edge weights are in the confidence ellipsoids $\{\mcC_{i,a}(t),i\in[N]\}$, under that intervention. Due to the linear structure, for any intervention $a\in\mcA$, the UCB is defined as
\begin{equation}
\label{eq:ucb_definition}
    {\rm UCB}_a(t) \;\triangleq \max_{\{\forall i\in[N]: [\Theta]_i\in \mcC_{i,a}(t)\}}\inner{f(\Theta)}{\bnu }\ ,
\end{equation}
where we define $f(\Theta)\triangleq \sum_{\ell = 0}^L [\Theta^\ell]_N$ and $\Theta^{\ell}$ is the $\ell$-th power of matrix $\Theta$.

Based on the UCB in~\eqref{eq:ucb_definition}, at time $t$, our algorithm selects the intervention that maximizes the UCB, 
\begin{equation}
    a(t)\;=\;\argmax_{a \in \mcA} {\rm UCB}_a(t)\ .
\end{equation}

\section{Regret Analysis}
\label{sec:regret}
In this section, we present the performance guarantees for the proposed \algonameRWLinSEMUCB~algorithm. To derive the upper bound, we start by providing a concentration bound for the W-OLS estimator. Notably, we investigate a vector norm that differs from existing work in robust bandits. This norm was first investigated in \cite{russac2019weighted} under the non-stationary setting, and our investigation builds on this to provide novel insights into the robust behavior of the W-OLS.

\begin{lemma}[Estimator concentration]\cite[Lemma 2]{yan2023robust}
\label{lem:beta(t)indeviation}
Under a deviation budget $C$, with a probability at least $1-2\delta$, for any node $i\in[N]$ and $t\geq 0$, we have
\begin{align}
        \|[\bB(t)]_i-\bB_i\|_{\bV_i(t) [\widetilde{\bV}_i(t)]^{-1} \bV_i(t)}  \;&\leq\; \beta(t,\delta)\ ,\\
        \mbox{and}\quad\|[\bB^*(t)]_i-\bB^*_i\|_{\bV_i^*(t) [\widetilde{\bV}_i^*(t)]^{-1} \bV_i^*(t)} \;&\leq\; \beta(t,\delta)\ ,
    \end{align}
where $\beta(t,\delta) \triangleq \sqrt{2\log\left(1/\delta\right)+d\log\left(1+m^2t/dC^2\right)} + 1 + m$ denotes the confidence radius.
\end{lemma}

\begin{figure*}
\centering
    \begin{subfigure}{.31\textwidth}
        \centering
        \includegraphics[height=3.2 cm]{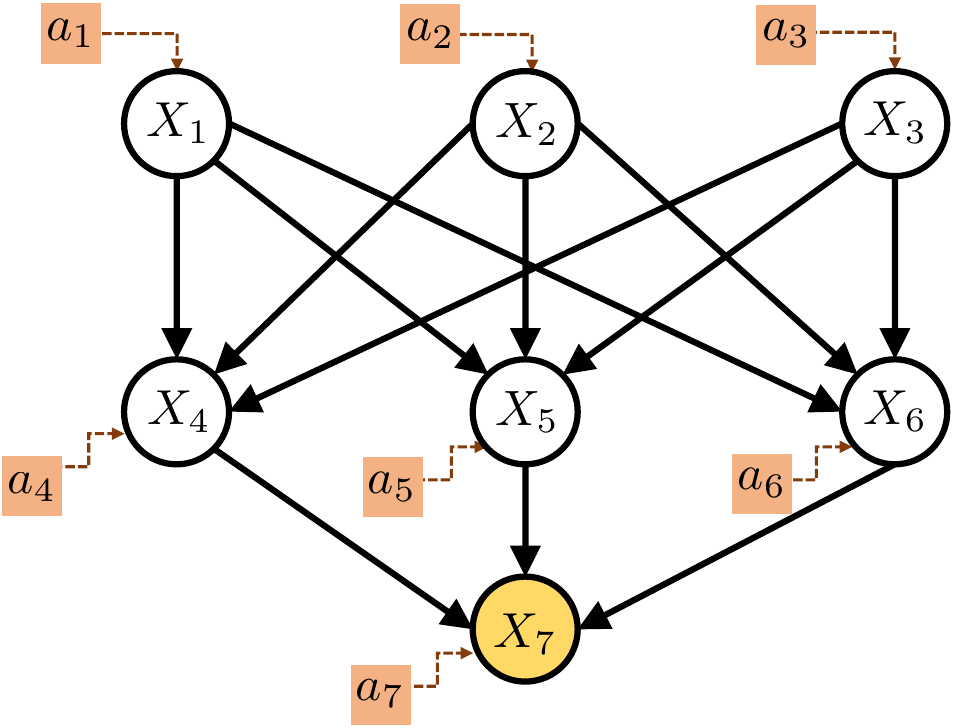}
        \caption{Hierarchical graph ($d=3$,$L=2$).}
        \label{fig:he_example}
    \end{subfigure}
    \begin{subfigure}{0.31\textwidth}
        \centering
        \includegraphics[height = 3.2 cm]{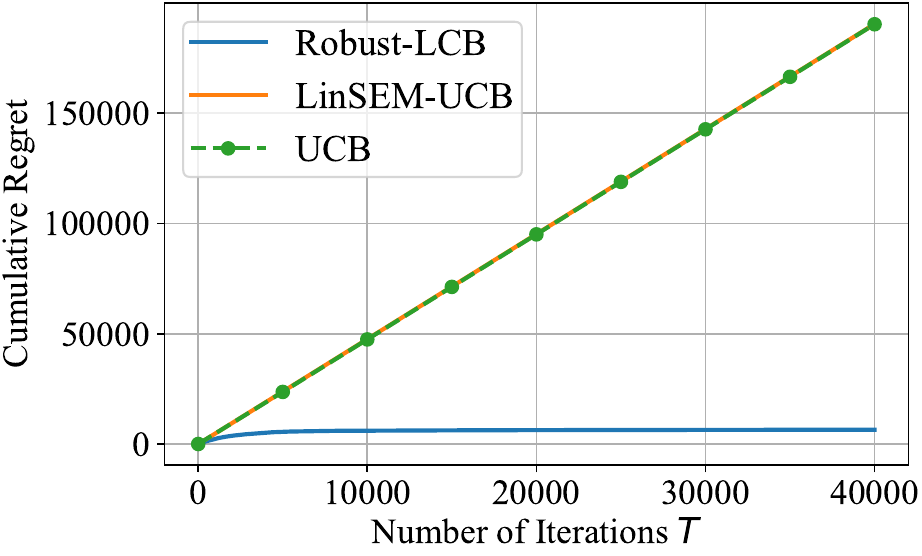}
        \caption{Cumulative regret when $C=\sqrt{T}$.}
        \label{fig:he_regret}
    \end{subfigure}
        \begin{subfigure}{0.31\textwidth}
        \centering
        \includegraphics[height=3.2 cm]{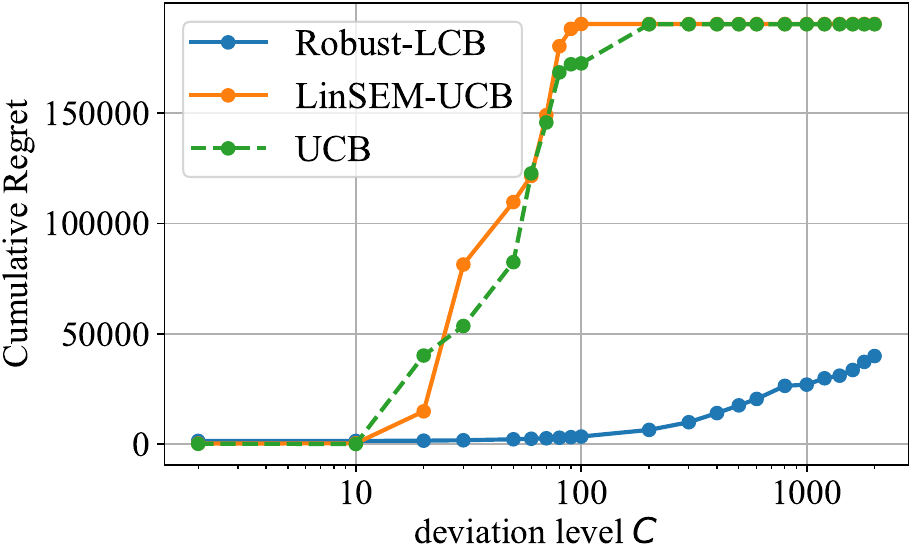}
        \caption{Cumulative regret at $T = 40000$.}
        \label{fig:he_C}
\end{subfigure}
\caption{Experiment result on hierarchical graph}
    \label{fig:diff_c}
\end{figure*}

This lemma offers high probability error bounds for estimators. Due to the causal structure, these errors accumulate and propagate along the causal paths, leading to the reward node $N$. Consequently, we analyze the compounding impacts of estimation errors and model deviations.
This analysis involves examining the eigenvalues of the weighted Gram matrices $\bV_{i,a(t)}(t)$ and $\widetilde{\bV}_{i,a(t)}(t)$. Finally, we result in the following regret upper bound. 

\begin{theorem}[Regret upper bound]
\label{thm:measure2}
 Under a deviation budget $C$, by setting $\delta=\frac{1}{2NT}$ and confidence radius $\beta(t,\delta)$ according to Lemma~\ref{lem:beta(t)indeviation}, the average cumulative regret of \algonameRWLinSEMUCB\  is upper bounded by
\begin{equation}
\label{equ:upperbound}
    \E \left[R(T)\right]\leq 2m+\tilde{\mathcal{O}}\left(d^{L-\frac{1}{2}}(\sqrt{T} + C)\right) \ .
\end{equation}
\end{theorem}
\begin{proof}
    See Appendix~\ref{proof:thm:measure2}.
\end{proof}

The regret bound established in Theorem~\ref{thm:measure2} can be decomposed into two parts. The first term recovers the order of the optimal rate achieved in the time-invariant setting. The second term captures the impact of model deviation on the regret bound, that is, the cost of handling unknown model fluctuations. This upper bound improves that of \cite{yan2023robust}. We also remark that by setting $C=1$ in the algorithm for time-invariant setting, our regret bound reduces to $\tilde\mcO(d^{L-\frac{1}{2}}\sqrt{T})$, which improves the regret bound for time-invariant setting in~\cite{varici2022causal} by a factor of $\sqrt{N}$. Next, we present a lower bound that confirms the tightness and optimality of our upper bound.  

\begin{theorem}[Regret lower bound]\label{thm:lowerbound}\cite[Corollary~9]{yan2023robust}
    For any degree $d$ and graph length $L$ and any algorithm, there exists a causal bandit instance such that the expected regret is at least
\begin{align}
    \E\left[R(T)\right]\geq \Omega(d^{\frac{L}{2}-2}\max\{\sqrt{T}\; , \; d^2C\})\ .
\end{align}
\end{theorem}
Comparing the results in Theorem~\ref{thm:measure2} and Theorem~\ref{thm:lowerbound}, the regret bound of the Robust-LCB algorithm is tight in terms of $T$ and $C$. The only gap between the upper and lower bound is a mismatch in the exact order of the exponential scaling of degree $d$ with graph length $L$. The mismatch in the exact order between $d^L$ and $d^{\frac{L}{2}}$ in bounds exists in all the relevant literature, even in simpler settings. For instance, consider linear bandits with time-invariant models with dimension $d$ (a special case of our model by setting $L=1$ and no model variations). For the widely used optimism in the face of uncertainty linear bandit algorithm (OFUL) in \cite{abbasi2011improved}, the lower and upper bounds behave according to $\tilde \mcO(\sqrt{dT})$ and $\tilde \mcO (d\sqrt{T})$, respectively. This gap in terms of matches exactly our gap. This gap represents potential avenues for future research.

\section{Empirical Evaluations}
\label{sec:experiment}

In this section, we assess the robustness of the Robust-LCB algorithm. To the best of our knowledge, there is no baseline CB algorithm that can be used as a natural baseline for performance comparisons. Furthermore, soft interventions on continuous variables of a CB model are implemented by only \algonameLinSEMUCB~of \cite{varici2022causal}. Therefore, to assess the robustness, we compare our \algonameRWLinSEMUCB~algorithm with \algonameLinSEMUCB~and the standard non-causal UCB algorithm.

\vspace{.08 in}
\noindent\textbf{Parameter setting:} We consider the hierarchical graph (Figure~\ref{fig:he_example}). The noise terms are uniformly sampled from $[0,2]$.
We set the observational and interventional weights for node $i\in[N]$ to $0.5/\sqrt{|\Pa(i)|}$ and $1/\sqrt{|\Pa(i)|}$, respectively.  We set the number of nodes in Layer 2 to $3$ and Layer 1 to $9$. 
We let the deviations on the model occur at earlier rounds to simulate the worst-case scenario for a given deviation level $C$.
When a deviation occurs on node $i\in[N]$, the weights are deliberately altered to change the optimal action, thereby challenging the algorithm's performance. The simulations are repeated $100$ times, and the average cumulative regret is reported.

\vspace{.08 in}
\noindent\textbf{Comparison of the bounds.} Figure~\ref{fig:he_regret} compares the cumulative regret of \algonameRWLinSEMUCB\ with that of \algonameLinSEMUCB\ and UCB under a model deviation level of $C =\sqrt{T}$. It demonstrates that only \algonameRWLinSEMUCB\  achieves sub-linear regret, whereas the other two algorithms incur linear regret. Furthermore, it is noteworthy that the LinSEM-UCB exhibits nearly the worst possible regret as the design of the deviation also showcases the worst case for LinSEM-UCB. In contrast, the regret of UCB tends to be the worst possible outcome as the graph's complexity increases even when the deviation is not designed for it. These findings imply that the estimators of these algorithms become ineffective when faced with such deviations, resulting in the selection of a sub-optimal (possibly the worst) arm.

\vspace{.1 in}
\noindent\textbf{Robustness against $C$.} Figure~\ref{fig:he_C} plots the cumulative regret at $T=40000$ when the model deviation budget $C$ changes from $2$ to $2000$ for all algorithms. 
We observe that \algonameLinSEMUCB\  and UCB perform slightly better than our algorithm when there is almost no deviation ($C\approx 0$). This can be viewed as the compromise needed to guarantee the robustness of our algorithm. However, both \algonameLinSEMUCB\ and UCB begin to fail with even a minor model deviation, as small as $C= 15$. Furthermore, they tend to reach nearly the worst possible regret when the deviation level rises high enough ($C=200$). In comparison, \algonameRWLinSEMUCB\ outperforms when the deviation is more than negligible, as its regret scales sub-linear with the deviation level $C$.

\vspace{0.05 in}
\section{Conclusion}
\vspace{0.02 in}
In this paper, we have studied the sequential design of interventions over graphical causal models where both the observational and interventional models are unknown and undergo temporal variations. Our focus is on the general soft intervention model, and we proposed an algorithm to optimize the reward over the graph while ensuring robustness against model variations. The proposed algorithm maintains sub-linear or nearly optimal regret under a wide range of model deviations. In addition to handling the model deviations, the established upper bound also improves upon existing results in the time-invariant setting. The regret bound depends on the graph structure only through its parameters: maximum degree $d$ and the length of the longest causal path $L$, and is independent of the number of nodes $N$.
\newpage

\bibliographystyle{IEEEtran}
\bibliography{main}

\clearpage
\newpage
\appendices

\section{Additional Notations}
First, we provide notations that are useful in our analyses. Since we are dealing with matrices, we denote the singular values of a matrix $\bA\in\R^{M \times N}$ with $M\geq N$, by
\begin{align}
    \sigma_{1}(\bA) &\geq \sigma_{2}(\bA) \geq \dots \geq \sigma_{N}(\bA) \ .
\end{align}

In the proof, we often work with zero-padded vectors and corresponding matrices. As a result, the matrices that contain these vectors have non-trivial \emph{null space} leading to zero singular values. In such cases, we use the \emph{effective} smallest singular value that is non-zero. We denote the \emph{effective} largest and smallest eigenvalues that correspond to effective dimensions of a positive semidefinite matrix $\bA$ with rank $k$ by
\begin{equation}
    \smax{\bA} \triangleq \sigma_{1}(\bA) \ , \quad \mbox{and} \quad \smin{\bA} \triangleq \sigma_{k}(\bA) \ .
\end{equation}
For a square matrix $\bU = \bA \bA^{\top} \in \mathbb{R}^{N \times N}$, we denote the \emph{effective} largest and smallest eigenvalues by\footnote{For matrix $\bV=\bU+\bI$, we denote the \emph{effective} smallest eigenvalues by $\lmin{\bV} \triangleq \sigma_{\min}^2(\bA) +1$.}
\begin{equation}
        \lmax{\bU} \triangleq  \sigma_{\max}^2(\bA) \ ,
    \quad \mbox{and} \quad \lmin{\bU} \triangleq  \sigma_{\min}^2(\bA) \ . 
\end{equation}
Then we construct data matrices that are highly related to Gram matrices. At time~$t\in\N$ and for any node $i\in[N]$, the data matrices $\bU_{i}(t)\in\R^{t \times N}$ and $\bU^{*}_{i}(t)\in\R^{t \times N}$ consist of the weighted observational and interventional data, respectively. Specifically, for any $s\in[t]$ and $i\in[N]$, we define
\begin{align}\label{eq:D_it_obs}
    \big[\bU_{i}^{\top}(t)\big]_{s} &\;\triangleq \;\mathbbm{1}{\{i \notin a(s)\}} \sqrt{w_i(s)} \bX_{\Pa(i)}^{\top}(s) \ , \\ 
\mbox{and} \quad     \big[{\bU^{*}_{i}}^{\top}(t)\big]_{s}& \;\triangleq\; \mathbbm{1}{\{i \in a(s)\}} \sqrt{w_i(s)}\bX_{\Pa(i)}^{\top}(s) \ .
\end{align}
Similarly to \eqref{eq:Ba_construct}, we denote the relevant data matrices for node $i\in[N]$ under intervention $a \in \mcA$ by
\begin{align}
    \bU_{i,a}(t) &\;\triangleq\; \mathbbm{1}{\{i \in a\}} \bU^{*}_{i}(t) +  \mathbbm{1}{\{i \notin a\}} \bU_{i}(t) \ ,  \label{eq:D_ita}\\
    \bV_{i,a}(t) &\;\triangleq\; \mathbbm{1}{\{i \in a\}} \bV^{*}_{i}(t) +  \mathbbm{1}{\{i \notin a\}} \bV_{i}(t) \ . \label{eq:V_ita}
\end{align}

Define $N^{*}_{i}(t)$ as the number of times that node $i\in[N]$ is intervened, and $N_i(t)$ as its complement, i.e.,
\begin{align} \label{eq:def_N_it}
    &N^{*}_{i}(t) \;\triangleq\; \sum_{s=1}^t \mathbbm{1}{\{i \in a(s)\}}\ , \\ \mbox{and} \quad &
    N_{i}(t) \;\triangleq\; t -  N^{*}_{i}(t)\ .
\end{align}
Accordingly, for any $i\in[N]$ and $t\in\N$, define
\begin{align} \label{eq:def_N_iat}
    N_{i,a}(t) &\triangleq \mathbbm{1}\{i \in a\} N^{*}_{i}(t) + \mathbbm{1}\{i \notin a\} N_{i}(t) \ , 
\end{align}

Furthermore, we define the \emph{causal depth} of node $i$ as the length of the longest directed causal path that ends at node $i\in[N]$ and denote it by $L_i$.

To proceed, we define the second-moment matrices and its \emph{effective} largest and smallest eigenvalues as
\begin{align}
\label{equ:kappa}
    \Sigma_{i,a}(t) &\triangleq \E_{\bX \sim \P_{a}}\left[ \bX_{\Pa(i)}(t)\bX_{\Pa(i)}^\top(t)\right]\ , \\
    \kappa_{\min} &\triangleq \min_{i\in[N], a\in\mcA,t\in[T]}\smin{\Sigma_{i,a}(t)} \ , \\
    \kappa_{\max} &\triangleq \max_{i\in[N], a\in\mcA,t\in[T]}\smax{\Sigma_{i,a}(t)} \ ,
\end{align}
where $\kappa_{\min}>0$ is guaranteed since there is no deterministic relation between nodes and their patients. These variables are inherent to the system and remain unknown to the learner. Given our focus on the weighted OLS estimator, we also introduce singular values related to auxiliary variables $\bX'_{\Pa(i)}(t)\triangleq\sqrt{w_i(t)}\bX_{\Pa(i)}(t)$ and $\widetilde{\bX}_{\Pa(i)}(t)\triangleq w_i(t)\bX_{\Pa(i)}(t)$. Accordingly, we define the second weighted moment matrices as follows.
\begin{align}
    \Sigma'_{i,a,w_{i}}(t)&\triangleq \E_{\bX \sim \P_{a}}\left[w_i(t) \bX_{\Pa(i)}(t) \bX_{\Pa(i)}^\top(t)\right]\ ,\\ 
    \widetilde{\Sigma}_{i,a,w_{i}}(t)&\triangleq \E_{\bX \sim \P_{a}}\left[w_i^2(t) \bX_{\Pa(i)}(t) \bX_{\Pa(i)}^\top(t)\right]\ . \label{eq:second_moment_definition_V}
\end{align} 

Lastly, we define $\widetilde \bB_{a}$ as the matrix that attains ${\rm UCB}_{a}(t)$, i.e., 
\begin{equation}
    \widetilde \bB_{a}=\argmax_{[\Theta]_i\in \mcC_{i,a}(t)}\inner{f(\Theta)}{\bnu }\ .
\end{equation}
Accordingly, we define the auxiliary variable $\widetilde \bX(t)$ generated according to the following SEM
\begin{equation}
    \widetilde \bX(t) = \widetilde \bB_{a(t)}^{\top} \widetilde \bX(t) + \epsilon(t) \ .
\end{equation}

\section{Cumulative Estimation Error}
\begin{lemma}
\label{lem:estimation error}
    If $[\bB]_i\in \mcC_i(t)$ and $[\bB]^*_i\in \mcC^*_i(t)$ for all $t\in[T]$ and $i\in[N]$, then for $i\in[N]$ we have
    \begin{equation}
        \sum_{t=1}^{T} \E_{a(t)}\left|\widetilde X_i(t) -X_i(t)\right| 
        \leq 2 m \beta_T \mcB \sum_{\ell=1}^{L_i} d^{\ell-1}  \ ,
    \end{equation}
    where we define the term
    \begin{align}
        \mcB &= \frac{4\sqrt{m \kappa_{\max}}}{\kappa_{\min}}\sqrt{T} + \frac{8}{\kappa_{\min}}\sqrt[4]{\frac{3T}{2}}  +E_1  \\
      &\quad + \frac{4m}{\kappa_{\min}}\log\left(\frac{\kappa_{\min}}{m}\sqrt{\frac{T}{2}}+\alpha m^2\right) C \ ,
    \end{align}
    where $\tau \triangleq  \frac{\alpha^2 m^6}{\kappa_{\min}^2}$, $\alpha=\sqrt{\frac{16}{3}\log((d+1)T^{5/2}(T+1))}$ and
    \begin{align}
    E_1&= 4\frac{\sqrt{m \kappa_{\max}}}{\kappa_{\min}} \sqrt{\tau} \log\left(\sqrt{\frac{T}{2}} + \sqrt{\tau}\right) \\
    &\quad + 4\sqrt{\frac{\alpha m^5}{\kappa_{\min}^3}} \log \left(\frac{\sqrt{\frac{1}{\tau}}\sqrt[4]{\frac{T}{2}}+\sqrt[4]{4}+1}{\sqrt{\frac{1}{\tau}}\sqrt[4]{\frac{T}{2}}+\sqrt[4]{4}-1}\right)\\
    &\quad + 8 \tau \left(\frac{1}{C }\sqrt{\kappa_{\max}\tau+\alpha m^2 \sqrt{\tau}} +1\right) \\
    & \quad + \frac{m}{C T} + \frac{2m}{3C } +   1 \ .
\end{align}
\end{lemma}

\begin{proof}
    We establish this via induction on the causal depth $L_i$.

    \vspace{0.1 in}
    \textbf{Base step: $L_i=1$.} For node $i\in[N]$ with causal depth $L_i=1$, we show that
    \begin{equation}
    \sum_{t=1}^{T}\E_{a(t)}\left| \widetilde X_{i}(t) - X_{i}(t)\right|  \leq 2m \beta_T \mcB\ .
    \end{equation}
    When the causal path of a node is $L_i=1$, according to SEM defined in \eqref{eq:linear model}, we have the following expansion:
\begin{align}
   &\sum_{t=1}^{T} \E_{a(t)} \left| \widetilde X_{i}(t) -  X_{i}(t)\right| \\
   &= \sum_{t=1}^{T}  \E_{a(t)} \left| [\widetilde \bB_{a(t)}(t)]^{\top} X_{\Pa(i)} - [\bB_{a(t)}]^{\top} X_{\Pa(i)} \right| \\
   \label{base_mid_0}&\leq \sum_{t=1}^{T}  \E_{a(t)} \sup_{b_1,b_2 \in \mcC_{i,a(t)}(t)} (b_1-b_2)^{\top} X_{\Pa(i)} \\
   \label{base_mid_1}& \leq \sum_{t=1}^{T}  \E_{a(t)} \!\!\!  \sup_{b_1,b_2 \in \mcC_{i,a(t)}(t)} \!\!\! \norm{b_1-b_2}_{\bV_{i,a(t)}(t) [\widetilde{\bV}_{i,a(t)}(t)]^{-1} \bV_{i,a(t)}(t)}\\
   & \qquad \times \norm{X_{\Pa(i)}}_{\big[\bV_{i,a(t)}(t) [\widetilde{\bV}_{i,a(t)}(t)]^{-1} \bV_{i,a(t)}(t)\big]^{-1}}\\
   \label{base_mid_2}&\leq 2 m \beta_T\\
   & \qquad \times \sum_{t=1}^{T} \lambda_{\min}^{-1/2}(\bV_{i,a(t)}(t) [\widetilde{\bV}_{i,a(t)}(t)]^{-1} \bV_{i,a(t)}(t)) \ ,
\end{align}
where \eqref{base_mid_0} holds due to $\widetilde \bB_{a(t)}(t)$ and $\bB_{a(t)}(t)$ lies in the confidence ellipsoid $\mcC_{i, a(t)}(t)$, \eqref{base_mid_1} holds by Cauchy-Schwarz inequality and the supremum maintain inequality, and \eqref{base_mid_2} holds due to the definition of $\mcC_{i, a(t)}(t)$ and $\norm{X}_{\bV}\leq \norm{X}\lambda_{\min}^{-1/2}(\bV)$.

Now, when we define
\begin{equation}
\lambda_i(t) \triangleq \frac{\sqrt{\lmax{\widetilde{\bV}_{i,a(t)}(t)}}}{\lmin{\bV_{i,a(t)}(t)}} \ ,  
\end{equation}
The remaining is to bound $\sum_{t=1}^{T}\lambda_i(t)$. To bound the singular value of the weighted second moment, we first need uniform bounds for the weights. We find a bound for the norm of $\|\bX_{\Pa(i)}(t)\|_{[\widetilde{\bV}_{i,a}(t)]^{-1}}$ across all $a\in \mcA$. This yields the following result.
\begin{equation}
\|\bX_{\Pa(i)}(t)\|_{[\widetilde{\bV}_{i,a(t)}(t)]^{-1}}\leq \frac{\|\bX_{\Pa(i)}(t)\|}{\lambda^{1/2}_{\min}(\widetilde{\bV}_{i,a(t)}(t))} \leq m \ .
\end{equation}
Then, the weights can be bounded by $\frac{1}{C m} \leq w_i(t) \leq \frac{1}{C}$.
Subsequently, we can bound the minimum and maximum singular values of matrices $\Sigma'_{i, a,w_{i}}(t)$ and $\widetilde{\Sigma}_{i, a,w_{i}}(t)$.
\begin{align}
    \kappa'_{\min} &\triangleq \frac{1}{C m}\kappa_{\min}\leq \smin{\Sigma'_{i,a,w_{i}}(t)} \ ,\\
    \kappa'_{\max} &\triangleq \frac{1}{C }\kappa_{\max}  \geq \smax{\Sigma'_{i,a,w_{i}}(t)}\ , \\
    \tilde{\kappa}_{\min} &\triangleq \frac{1}{C^2 m^2}\kappa_{\min}\leq \smin{\widetilde \Sigma_{i,a,w_{i}}(t)} \ , \\
    \tilde{\kappa}_{\max} & \triangleq \frac{1}{C^2}\kappa_{\max} \geq \smax{\widetilde \Sigma_{i,a,w_{i}}(t)}\ .
\end{align}
Moreover, we have $m' \triangleq \frac{1}{\sqrt{C }} m \geq \|\bX'_{\Pa(i)}(t)\|$ and  $\tilde{m} \triangleq \frac{1}{C} m \geq \|\widetilde{\bX}_{\Pa(i)}(t)\|$. In order to proceed, we need upper and lower bounds for the maximum and minimum singular values of $\bU_{i, a(t)}(t)$. However, these bounds depend on the number of non-zero rows of $\bU_{i, a(t)}(t)$ matrices, which equals to values of the random variable $N_{i, a(t)}(t)$.
Let us define the weighted constant
\begin{align}
    \gamma_n &\triangleq  \max\left\{\alpha m^2 \sqrt{n},\alpha^2 m^2 \right\} \ ,  \label{eq:def_varepsilon_n_RW}\\
    \gamma'_n &\triangleq  \max\left\{\alpha m'^2 \sqrt{n},\alpha^2 m'^2 \right\} \ ,  \label{eq:def_varepsilon_n_prime_RW}\\
    \tilde{\gamma}_n &\triangleq  \max\left\{\alpha \tilde{m}^2 \sqrt{n},\alpha^2 \tilde{m}^2 \right\} \ , \quad \forall n \in [T] \ .  \label{eq:def_varepsilon_n_tilde_RW}
\end{align}
Then for every $t\in [T]$, and $n\in[t]$, we define the error events corresponding to the maximum and minimum singular values of $\bU_{i}(t)$ and $\widetilde{\bU}_{i}(t)$ as
\begin{align}
    \notag \mcE_{i,n}(t) \triangleq& \Bigg\{ N_{i}(t) = n \quad \text{and} \\
    & \notag\left\{\smin{\bU_{i}(t)}\leq \sqrt{\max \left \{0, n \kappa'_{\min}- \gamma'_n\right\}}\right. \\ 
    & \hspace{-0.in}  \  \text{or} \  \left. \smax{\bU_{i}(t)}\geq \sqrt{n \kappa'_{\max} + \gamma'_n}  \right\} \Bigg\} \ , \label{eq:def_error_int_obs_RW} \\
    \notag \widetilde{\mcE}_{i,n}(t) \triangleq& \Bigg\{ N_{i}(t) = n \quad \text{and} \\
    & \notag\left\{\smin{\widetilde{\bU}_{i}(t)}\leq \sqrt{ \max\left\{0, n \tilde{\kappa}_{\min} - \tilde{\gamma}_n\right\}}\right. \\ 
    & \hspace{-0.in}  \  \text{or} \  \left. \smax{\widetilde{\bU}_{i}(t)}\geq \sqrt{n \tilde{\kappa}_{\max} + \tilde{\gamma}_n}  \right\} \Bigg\} \ , \label{eq:def_error_int_obs_tilde_RW}
\end{align}
Similarly, we can define $\mcE^{*}_{i,n}(t)$ and $\widetilde{\mcE}^{*}_{i,n}(t)$ by replacing $N_i(t)$ and $\bU_i(t)$ (or $\widetilde{\bU}_{i}(t)$) by $N^{*}_{i}(t)$ and $\widetilde \bU_i(t)$ (or $\widetilde{\bU}^{*}_{i}(t)$), respectively.

\begin{lemma} \cite[Lemma~8]{yan2023robust} \label{lm:error_event_int_RW}
The probability of the error events $\mcE_{i,n}(t)$, $\mcE^{*}_{i,n}(t)$, $\widetilde{\mcE}_{i,n}(t)$ and $\widetilde{\mcE}^{*}_{i,n}(t)$ are upper bounded as
\begin{align}
    &\notag\max\left\{\P(\mcE_{i,n}(t)), \P(\mcE^{*}_{i,n}(t)), \P(\widetilde{\mcE}_{i,n}(t)), \P(\widetilde{\mcE}^{*}_{i,n}(t))\right\}\\
    & \qquad \leq d \exp \left( -\frac{3\alpha^2}{16} \right) \ .
\end{align}
\end{lemma}

Then we define the union error event $\mcE_{i,\cup}$ as 
\begin{equation} \label{eq:def_union_error_singular}
    \mcE_{i,\cup} \triangleq \{ \exists\ (t,n) : t \in [T], n \in [t], \ \mcE_{i,n}(t) \ \text{or} \  \mcE^{*}_{i,n}(t)  \} \ .
\end{equation}
By taking a union bound and using Lemma \ref{lm:error_event_int_RW}, we have
\begin{align}
    \P(\mcE_{i,\cup}) &\leq \sum_{i=1}^{N} \sum_{t=1}^{T}\sum_{n=1}^{t} 4 d \exp \left( -\frac{3\alpha^2}{16} \right)\\
    &\leq 2 N T (T+1) d \exp \left( -\frac{3\alpha^2}{16} \right) \ . \label{eq:def_union_error_singular_prob_RW}
\end{align}

Now we turn back to $\E\left[ \sum_{t=1}^T \lambda_i(t) \right]$ to analyze it under the complementary events $\mcE_{i,\cup}$ and $\mcE_{i,\cup}^{\C}$.\\
\textbf{Bounding term $\E \left[\mathbbm{1}\{\mcE_{i,\cup}\} \sum_{t=1}^T \lambda(t)\right]$.}

Since $\lmin{\bV_{i,a(t)}(t)}\geq 1$, we have the following unconditional upper bound.
\begin{align}
    \lambda_i(t) &=  \frac{\sqrt{\lmax{\widetilde{\bV}_{i,a(t)}(t)}}}{\lmin{\bV_{i,a(t)}(t)}}  \leq \sqrt{\lmax{\widetilde{\bV}_{i,a(t)}(t)}}  \label{eq:big_bounding_error_1_RW}\\
    &\leq \sqrt{1 + \frac{1}{C ^2} \sum_{s=1}^{t} \lmax{\bX_{\Pa(i)}(s) \bX_{\Pa(i)}^{\top}(s)} }\label{eq:big_bounding_error_1mid_RW}\\
    &=  \sqrt{1 + \frac{1}{C ^2} \sum_{s=1}^{t} \norm{\bX_{\Pa(i)}(s)}^2} \\
    &\leq \sqrt{\frac{m^2}{C^2}t + 1} \ ,  \label{eq:big_bounding_error_2_RW}
\end{align}
where~\eqref{eq:big_bounding_error_1mid_RW} holds due to the expansion of weighted gram matrix and ~\eqref{eq:big_bounding_error_2_RW} follows from the fact that $\norm{\bX}\leq m$. Hence, we have
\begin{align}
    \E \left[\mathbbm{1}\{\mcE_{i,\cup}\} \sum_{t=1}^T \lambda_i(t) \right]
    &\overset{\eqref{eq:big_bounding_error_2_RW}}{\leq} \E \left[\mathbbm{1}\{\mcE_{i,\cup}\} \sum_{t=1}^T \sqrt{ \frac{m^2}{C ^2}t + 1}  \right] \\
    &= \P(\mcE_{i,\cup})\sum_{t=1}^T \sqrt{ \frac{m^2}{C ^2}t + 1} \ . \label{eq:big_bounding_error_3_RW}
\end{align}
Furthermore, the summation part is bounded as
\begin{align}
    \sum_{t=1}^T \sqrt{ \frac{m^2}{C ^2}t + 1} 
    &\leq \frac{m}{C }\sqrt{T} + T + \int_{t=1}^{T} \frac{m}{C }\sqrt{t} dt \\
    &= \frac{m}{C }\sqrt{T} + T + \frac{2m}{3C }(T^{3/2}-1) \ . \label{eq:big_bounding_error_4_RW}
\end{align}
By setting $\alpha = \sqrt{\frac{16}{3}\log(2d N T^{5/2}(T+1))}$, we obtain
\begin{align}
     &\E \left[\mathbbm{1}\{\mcE_{i,\cup}\} \sum_{t=1}^T \sqrt{m^2 t + 1}  \right] \overset{\eqref{eq:big_bounding_error_3_RW}}{\leq} \P(\mcE_{i,\cup}) \sum_{t=1}^{T}\sqrt{m^2t+1} \\
     &\overset{\eqref{eq:def_union_error_singular_prob_RW}}{\leq} \underset{= T^{-3/2}}{\underbrace{\frac{N T (T+1) d}{\exp(\log(d N T^{5/2} (T+1)))}}}  \sum_{t=1}^{T}\sqrt{m^2t+1}  \\
     &\overset{\eqref{eq:big_bounding_error_4_RW}}{\leq} T^{-3/2} \left(\frac{m}{C }\sqrt{T} + T+ \frac{2m}{3C }(T^{3/2}-1)\right) \\
     & < \frac{m}{C T} + \frac{2m}{3C } +   1 \ . \label{eq:error_event_bound_RW}
\end{align}
\textbf{Bounding $\E \left[\mathbbm{1}\{\mcE_{i,\cup}^\C\} \sum_{t=1}^T \lambda_i(t) \right]$.}
Considering the event $\mcE_{i,\cup}^\C$, we can use the following bounds on the singular values
\begin{align}
    \hspace{-0.15 in}\smin{\bU_{i,a(t)}(t)} &\geq  \sqrt{\max\left\{0, N_{i,a(t)}(t)\kappa'_{\min}- \gamma'_n\right\}}  \label{eq:smin_D_ita_RW} \ ,  \\
    \hspace{-0.15 in}\smax{\widetilde{\bU}_{i,a(t)}(t)} &\leq \sqrt{N_{i,a(t)}(t) \tilde{\kappa}_{\max}+ \tilde{\gamma}_n}\label{eq:smax_D_ita_RW}  \ .
\end{align}
Thus, the targeted sum can be upper-bounded by
\begin{align}
    &\E \left[\mathbbm{1}\{\mcE_{i,\cup}^\C\} \sum_{t=1}^T\lambda(t)\right] \\
    &= \E \left[\mathbbm{1}\{\mcE_{i,\cup}^\C\} \sum_{t=1}^T \frac{\sqrt{\lmax{\widetilde{\bV}_{i,a(t)}(t)}}}{\lmin{\bV_{i,a(t)}(t)  }}\right]  \\
    &= \E \left[\mathbbm{1}\{\mcE_{i,\cup}^\C\} \sum_{t=1}^T   \frac{\sqrt{\smaxx{\widetilde{\bU}_{i,a(t)}(t)}{2}+1}}{\sminn{\bU_{i,a(t)}(t)}{2}+1}\right] \\
    &\leq \E  \sum_{t=1}^T  \left[ \frac{\sqrt{N_{i,a(t)}(t) \tilde{\kappa}_{\max}+ \tilde{\gamma}_n+1} }{\max\left\{0, N_{i,a(t)}(t)\kappa'_{\min}- \gamma'_n\right\}+1} \right] \ .
\end{align}
It is worth noting that the term in the summation has a critical point, and we bound the two regions separately. To initiate this process, we define the function $h(x)$ as
\begin{equation}
    h(x) \triangleq \frac{\sqrt{x \tilde{\kappa}_{\max}+ \tilde{\gamma}_n+1} }{\max\left\{0, x\kappa'_{\min}- \gamma'_n\right\}+1}\ , \quad x>0  \label{eq:h_function_RW} \ .
\end{equation}
In order to analyze the behavior of the function $h$, we introduce $\tau\triangleq\frac{\alpha^2m^6}{\kappa_{\min}^2}$ as the critical point. Note that when $x\leq \tau$, we have $x\kappa'_{\min} < \gamma_n$. In this case, $h(x)$ is equal to
\begin{equation}
    h(x) = \sqrt{x \tilde{\kappa}_{\max}+ \tilde{\gamma}_n+1}  \label{eq:h_function_small} \ ,
\end{equation}
which is an increasing function over the region. To upper bound the $h$ function when $x>\tau$, we define the  $g$ function when $x>\tau$ as follows.
\begin{equation}
    g(x) \triangleq  \frac{\sqrt{x \kappa_{\max}+ \alpha m^2\sqrt{x}}}{ x\kappa_{\min}/m- \alpha m^2\sqrt{x}} + \frac{C }{ x\kappa_{\min}/m- \alpha m^2\sqrt{x}} \ .
\end{equation}
We have the following theorem to show the monotonicity and relation of $h(x)$ and $g(x)$.
\begin{lemma}\cite[Lemma~10]{yan2023robust}
$h(x)$ and $g(x)$ are both decreasing functions when $x>\tau$ and $h(x)<g(x)$, where $\tau$ is defined as $\frac{\alpha^2m^6}{\kappa_{\min}^2}$.
\end{lemma}

Now we are ready to bound the last term
\begin{equation}
     \E \left[\mathbbm{1}\{\mcE_{i,\cup}^\C\} \sum_{t=1}^T\lambda_i(t)\right]\leq \E \sum_{t=1}^{T}  h(N_{i,a(t)}(t))\ .
\end{equation}
We define the set of time indices at which the chosen actions are under-explored as
\begin{equation}
    \mcH_i \triangleq \left\{ t\in[T] \mid N_{i,a(t)}(t)\leq 4\tau  \right\}\ .
\end{equation}
It can be readily verified that $|\mcH_i|\leq 8\tau$. Furthermore, when $x\in\mcH_i$, we have
\begin{equation}
    h(x) \leq h(\tau) \leq \frac{1}{C }\sqrt{\kappa_{\max}\tau+\alpha m^2 \sqrt{\tau}} +1 \ , \ x \leq \tau \ .
\end{equation}
Then we can bound the summation when $\mcH_i$ occurs as follows.
\begin{align}
    &\E \sum_{t=1}^T \mathbbm{1}\{t\in \mcH_i\}  h(N_{i,a(t)}(t)) \\
    &\qquad \leq 8 \tau \left(\frac{1}{C }\sqrt{\kappa_{\max}\tau+\alpha m^2 \sqrt{\tau}} +1\right) \ .\label{mid:1}
\end{align}
Now we only need to bound the remaining part when $t \not\in \mcH_i$ 
\begin{equation}
    \E \sum_{t=1}^T \mathbbm{1}\{t\in \mcH_i^{\C}\} h(N_{i,a(t)}(t)) \ .
\end{equation}
Note that when $t\in \mcH_i^{\C}$, we have $N_{i,a(t)}(t)>\tau$ and
\begin{equation}
 h(N_{i,a(t)}(t))\leq  g(N_{i,a(t)}(t)) \ .
\end{equation}
Using the above results and noting that $g$ is a decreasing function, we obtain
\begin{align}
    &\sum_{t=1}^{T} \mathbbm{1}\{t\in \mcH_i^{\C}\}h(N_{i,a(t)}(t)) \\
    &\qquad \leq \sum_{t=1}^{T} \mathbbm{1}\{t\in \mcH_i^{\C}\} g(N_{i,a(t)}(t)) \\
    &\qquad \leq \sum_{n=4\tau+1}^{N_i(T)+4\tau} g(n) + \sum_{n=4\tau+1}^{N_i^*(T)+4\tau} g(n)\label{equ:G_bound_e:final}  \ .
\end{align}
We bound the discrete sums through integrals and define
\begin{equation}
    G_{\tau}(y) = \int_{x=4\tau}^{y} g(x)dx \ , \quad y\geq 4\tau \ . \label{eq:Hx_definition}
\end{equation}
Since $g(x)$ is a positive, non-increasing function, for any $k \in \mathbb{N}, k\geq 4\tau+1$ we have
\begin{equation}
    \sum_{n=4\tau+1}^{k} g(n) \leq \int_{x=4\tau}^{k} g(x)dx = G_{\tau}(k) \label{eq:h_side_4} \ .
\end{equation}
Then, the summation in \eqref{equ:G_bound_e:final} is upper bounded by
\begin{align}
    &\hspace{-0.1 in}\sum_{n=4\tau+1}^{N_i(T)+4\tau} g(n) +  \sum_{n=4\tau+1}^{N_i^*(T)+4\tau} g(n) \notag \\ &
    \hspace{0.5 in}\leq G_{\tau}(N_i(t)+4\tau) + G_{\tau}(N_i^*(t)+4\tau) \ . \label{eq:h_side_5}
\end{align}
Since $g(x)$ is positive and decreasing, and $G(y)$ is defined as an integral of the $g$ function with a positive first derivative and negative second derivative, it can be deduced that $G$ is a concave function. Thus, we have
\begin{align}
    &\hspace{-0.3 in}G_{\tau}(N_i(t)+4\tau) + G_{\tau}(N_i^*(t)+4\tau)\notag \\
    &\leq 2 G_{\tau}\left(\frac{T}{2}+4\tau \right) \ . \label{eq:h_side_6}
\end{align}
Next, we proceed to establish an upper bound for the function $G$, which can be upper bounded as
\begin{align}
    &G_{\tau}\left(\frac{T}{2}+4\tau\right) = \int_{x=4\tau}^{\frac{T}{2}+4\tau} g(x) {\rm d}x \\
    &\leq \int_{x=4\tau}^{\frac{T}{2}+4\tau} \sqrt{\frac{m^2\kappa_{\max}}{\kappa_{\min}}} \frac{1}{\sqrt{x\kappa_{\min}}-\sqrt{\tau\kappa_{\min}}}{\rm d}x \label{equ:G_bound:mid}\nonumber\\ 
    &\quad+ \int_{x=4\tau}^{\frac{T}{2}+4\tau}
    \sqrt{\alpha m^2(1+\frac{m\kappa_{\max}}{\kappa_{\min}})}\frac{x^{1/4}}{ x\kappa_{\min}/m- \alpha m^2\sqrt{x}} {\rm d}x \nonumber\\
    &\quad + \int_{x=4\tau}^{\frac{T}{2}+4\tau} \frac{C }{ x\kappa_{\min}/m- \alpha m^2\sqrt{x}} {\rm d}x \\ 
    &\leq 2\frac{\sqrt{m \kappa_{\max}}}{\kappa_{\min}}\left(\sqrt{\frac{T}{2}}+\sqrt{\tau} \log\left(\sqrt{\frac{T}{2}} + \sqrt{\tau}\right)\right)\nonumber\\
    &\quad+ \frac{4}{\kappa_{\min}}\sqrt[4]{\frac{T}{2}} + 2\sqrt{\frac{\alpha m^5}{\kappa_{\min}^3}} \log \left(\frac{\sqrt{\frac{1}{\tau}}\sqrt[4]{\frac{T}{2}}+\sqrt[4]{4}+1}{\sqrt{\frac{1}{\tau}}\sqrt[4]{\frac{T}{2}}+\sqrt[4]{4}-1}\right)\nonumber\\
    &\quad + \frac{2m\log\left(\frac{\kappa_{\min}}{m}\sqrt{\frac{T}{2}}+\alpha m^2\right)}{\kappa_{\min}} C  \ . \label{equ:G_bound_ec:final} 
\end{align}
where \eqref{equ:G_bound:mid} is due to the inequality $\sqrt{x+y}\leq \sqrt{x} + \sqrt{y}$ and $\sqrt{x-y} \geq \sqrt{x}-\frac{y}{\sqrt{x}}$ when $x\geq y$, and we use closed-form integral and discard positive terms in  \eqref{equ:G_bound_ec:final}. Combining the results in \eqref{eq:error_event_bound_RW}, \eqref{mid:1},  and \eqref{equ:G_bound_ec:final}, let $E_1$ denote the accumulation of terms that exhibit at most logarithmic growth rates with respect to $T$ and $C $.
\begin{align}
    E_1&= 4\frac{\sqrt{m \kappa_{\max}}}{\kappa_{\min}} \sqrt{\tau} \log\left(\sqrt{\frac{T}{2}} + \sqrt{\tau}\right) \\
    &\quad + 4\sqrt{\frac{\alpha m^5}{\kappa_{\min}^3}} \log \left(\frac{\sqrt{\frac{1}{\tau}}\sqrt[4]{\frac{T}{2}}+\sqrt[4]{4}+1}{\sqrt{\frac{1}{\tau}}\sqrt[4]{\frac{T}{2}}+\sqrt[4]{4}-1}\right)\\
    &\quad + 8 \tau \left(\frac{1}{C }\sqrt{\kappa_{\max}\tau+\alpha m^2 \sqrt{\tau}} +1\right) \\
    & \quad + \frac{m}{C T} + \frac{2m}{3C } +   1 \ .
\end{align}
Therefore, the final result for the bound is
\begin{align}
    \E\left[\sum_{t=1}^T  \lambda_i(t) 
      \right] &\leq \frac{4\sqrt{m \kappa_{\max}}}{\kappa_{\min}}\sqrt{T} + \frac{8}{\kappa_{\min}}\sqrt[4]{\frac{3T}{2}}  +E_1  \\
      & + \frac{4m}{\kappa_{\min}}\log\left(\frac{\kappa_{\min}}{m}\sqrt{\frac{T}{2}}+\alpha m^2\right) C \ .
      \label{eq:accumulation_E1}
\end{align}

Plugging~\eqref{eq:accumulation_E1} into~\eqref{base_mid_2}, we have
\begin{equation}
    \sum_{t=1}^{T}\E_{a(t)}\left| \widetilde X_{i}(t) - X_{i}(t)\right|  \leq 2m \beta_T \mcB \ .
\end{equation}

\textbf{Induction Step.} Assume that the property holds for causal depths up to $L_i=k$. We show that it will also satisfied for $L_i=k+1$. For this purpose, we start with the following expansion and apply the triangular inequality to find an upper bound for it.

\begin{align}
& \sum_{t=1}^{T}\E_{a(t)} \left|\widetilde X_{i}(t) - X_{i}(t)] \right| \\
\nonumber& = \sum_{t=1}^{T}  \E_{a(t)} \Big|  [\widetilde \bB_{a(t)}(t)]^{\top}_i \widetilde X_{\Pa(i)}(t) \\
& \qquad -  [\bB_{a(t)}]^{\top}_i X_{\Pa(i)}(t) \Big|\\
&=\nonumber \sum_{t=1}^{T} \E_{a(t)}  \Big| [\widetilde \bB_{a(t)}(t)]^{\top}_i \Big( \widetilde X_{\Pa(i)}(t) - X_{\Pa(i)}(t) \Big)\Big |\\
\label{equ:mid_of_lemma1} & \quad + \sum_{t=1}^{T} \E_{a(t)} \Big| \Big([\widetilde \bB_{a(t)}(t)]_i -  [ \bB_{a(t)}]_i\Big)^{\top} X_{\Pa(i)}(t) \Big|\\
&\leq \sum_{t=1}^{T}  \E_{a(t)} \norm{ \widetilde X_{\Pa(i)}(t) - X_{\Pa(i)}(t)} + 2m\beta_T \mcB \ \label{equ:mid_of_lemma} \\
    & \leq \sum_{j\in\Pa(i)} \sum_{t=1}^{T}  \E_{a(t)} \left|\widetilde X_{\Pa(i)}(t) - X_{\Pa(i)}(t)\right| \\
&\quad + 2m\beta_T \mcB \ .\label{equ:lemmasum2_mid}
\end{align}

where the transition to \eqref{equ:mid_of_lemma1} holds due to the triangular inequality via adding and subtracting terms $[\widetilde \bB_{a(t)}(t)]^{\top} X_{\Pa(i)}(t)$; \eqref{equ:mid_of_lemma} holds due to $\big\|\widetilde \bB_{a(t)}(t)\big\|\leq 1$ and similar proof as in the base step; and \eqref{equ:lemmasum2_mid} holds since the triangle inequality of $L_2$ norm.

Next, we find an upper bound on the first summand in \eqref{equ:lemmasum2_mid}. We notice that the summation is taken over all parents of node $i$, thus, we aim to find an upper bound for the error bound for each parent. Based on the induction assumption, for each node $j\in\Pa(i)$, we have 
\begin{align}
    \sum_{t=1}^{T}\E_{a(t)} \left|\widetilde X_{j}(t)- X_{j}(t)\right|
    & \;  \leq \; 2m\beta_T\mcB \sum_{\ell=1}^{L_j}  d^{\ell-1} \\
    &     \label{equ:inductioncondition} \;  \leq \; 2m\beta_T \mcB \sum_{\ell=1}^{k} d^{\ell-1} \ ,
\end{align}
where \eqref{equ:inductioncondition} holds since the causal depth of parent nodes is less than that of node $i$, i.e., $L_j\leq L_i-1 = k$ for all $j\in\Pa(i)$. Subsequently, plugging \eqref{equ:inductioncondition} into \eqref{equ:lemmasum2_mid}, we obtain
\begin{align}
    & \sum_{t=1}^{T} \sum_{j\in\Pa(i)} \E_{a(t)} \left|\widetilde X_{j}(t) - X_{j}(t)\right|\\
    &\quad\leq \; \sum_{j\in\Pa(i)} 2m\beta_T \mcB  \sum_{\ell=1}^{k}  d^{\ell-1}\\
    &\quad \leq \;  d \; 2m\beta_T \mcB  \sum_{\ell=1}^{k} d^{\ell-1}  \\
    &\quad \leq \; 2m\beta_T \mcB\sum_{\ell=1}^{k}  d^{\ell} \label{equ:lemma5.2sum2}\ ,
\end{align}
Combining  the results in~\eqref{equ:lemmasum2_mid}
and~\eqref{equ:lemma5.2sum2}, we conclude
\begin{equation}
\sum_{t=1}^{T} \E_{a(t)}\left|\widetilde X_i- X_i\right|   \leq 2m\beta_T \mcB  \sum_{\ell=1}^{k+1} d^{\ell-1} \ ,
\end{equation}
which proved the desired results.
\end{proof}

\section{Proof of Theorem~\ref{thm:measure2}}
\label{proof:thm:measure2}
We start by setting $\delta=\frac{1}{2NT}$ and defining the event in which over $T$ rounds, all confidence sets $\mcC_{i,t}$ contain the ground truth parameters $f_i$. Specifically, 
Now define the error events $\mcE_{i}$ and $ \mcE^{*}_{i}$ for $i\in[N]$ for each estimator 
\begin{align}
    \mcE_{i} &\triangleq \biggl\{ \forall t \in [T] : [\bB]_{i} \in \mcC_{i,t} \biggr\}  \ ,  \\
   \mcE^{*}_{i} &\triangleq \biggl\{ \forall t \in [T] : [\bB^{*}]_{i}\in \mcC^*_{i,t} \biggr\}  \ ,
\end{align}
where the $\beta(t,\delta)$ is chosen as
\begin{align}
&\beta\Big(t,\frac{1}{2NT}\Big) \\
&\hspace{0.04 in}\triangleq \sqrt{2\log\left(2NT\right)+d\log\left(1+m^2t/dC^2\right)} + 1 + m\ ,
\end{align}
Accordingly, define the event
\begin{equation}
    \mcE \triangleq \bigg(\bigcap_{i=1}^N  \mcE_{i} \bigg) \bigcap \bigg(\bigcap_{i=1}^N  \mcE^{*}_{i} \bigg)  \ .  \label{eq:ucb_conf_interval_event_union_RW}
\end{equation}
By invoking the union bound on probability and Lemma~\ref{lem:beta(t)indeviation}, we have
\begin{equation}
    \P(\mcE^{\C}) \;  \leq \; \sum_{i=1}^{N} \Big(\P(\mcE^{\C}_i) + \P(\mcE^{* \C}_i) \Big) \; \leq \; \sum_{i=1}^{N} \frac{1}{NT} \; = \; \frac{1}{T} \ .
\end{equation}
Next, we decompose the regret defined in \eqref{equ:regret}  under the events $\mcE$ and $\mcE^{\C}$.
\begin{align}
    &\E[R(T)] =  \sum_{t=1}^{T} \E_{a^*}X_{N}(t) - \E_{a(t)}X_{N}(t) \\
    \nonumber & \quad =  \sum_{t=1}^{T} \E \bigg[ \mathbbm{1}\{\mcE^{\rm c}\} \underset{\leq 2m}{\underbrace{\Big(\E_{a^*}[X_{N}(t)] - \E_{a(t)}[X_{N}(t)]\Big)}}\bigg]\\
    & \qquad  + \sum_{t=1}^{T}  \E \bigg[ \mathbbm{1}\{\mcE\}\;\Big(\underset{\leq \operatorname{UCB}_{a^*}(t)}{\underbrace{\E_{a^*}[X_{N}(t)]}} - \E_{a(t)}[X_{N}(t)]\Big)\bigg]\\
    &\quad  \leq 2m \sum_{t=1}^{T} \underset{\leq \frac{1}{T}}{\underbrace{\P (\mcE^{\rm c})}}
      + \sum_{t=1}^{T}  \E \bigg[ \operatorname{UCB}_{a(t)} - \E_{a(t)} [X_{N}(t))]\bigg]\\
    & \quad \leq 2m + \sum_{t=1}^{T} \E \bigg[ \E_{a(t)}[\widetilde X_{N}(t)] - \E_{a(t)}[X_{N}(t) ]\bigg]\label{eq:upperbound_plugin}\ ,
\end{align}
where we have used the set of inequalities 
\begin{align}
\E_{a^*}[X_N(t)] &\leq \operatorname{UCB}_{a^*}(t) \\
&\leq \operatorname{UCB}_{a(t)}(t) \leq \E_{a(t)}[\widetilde X_N(t)] \ .
\end{align}
Then using the lemma~\ref{lem:estimation error}, we have 
\begin{align}
    \E[R(T)]  &\leq 2m  +2 m \beta(T,\frac{1}{2NT}) \mcB \sum_{\ell=1}^{L} d^{\ell-1} \\
    & = 2m + \tilde \mcO\Big(d^{L-\frac{1}{2}} (\sqrt{T}+C)\Big)\label{eq:boundregret_1} \ ,
\end{align}
where the~\eqref{eq:boundregret_1} is due to the fact that for $d\geq 2$ we have
\begin{equation}
    \sum_{\ell=1}^{L} d^{\ell-1}=\frac{d^{L}-1}{d-1}   \leq 2d^{L-1} = \tilde{\mcO} (d^{L-1})\ .
\end{equation}

\end{document}